\documentclass[sigconf]{acmart}
\usepackage{amsmath}
\usepackage{amsthm}
\usepackage{mathtools}
\usepackage{booktabs} 
\usepackage[ruled,linesnumbered]{algorithm2e}
\usepackage{tabulary}
\usepackage{colortbl}
\usepackage{float}
\usepackage{subcaption}
\usepackage{xcolor}
\setcopyright{none}
\renewcommand\footnotetextcopyrightpermission[1]{}
\begin{document}
\title{Hypergraph Clustering: A Modularity Maximization Approach}

\author{Tarun Kumar}
\authornote{Both authors contributed equally to the paper}
\affiliation{%
  \institution{Robert Bosch Centre for Data Science and AI, IIT Madras}
}
\email{tkumar@cse.iitm.ac.in}

\author{Sankaran Vaidyanathan}
\authornotemark[1]
\affiliation{%
  \institution{Robert Bosch Centre for Data Science and AI, IIT Madras}
}
\email{sankaranv8@gmail.com}

\author{Harini Ananthapadmanabhan}
\authornote{Work done while the author was at IIT Madras}
\affiliation{%
  \institution{Google, Inc.}
}
\email{harini.nsa@gmail.com}

\author{Srinivasan Parthasarathy}
\affiliation{%
  \institution{The Ohio State University}
}
\email{srini@cse.ohio-state.edu}

\author{Balaraman Ravindran}
\affiliation{%
  \institution{Robert Bosch Centre for Data Science and AI, IIT Madras}
}
\email{ravi@cse.iitm.ac.in}

\begin{abstract}
Clustering on hypergraphs has been garnering increased attention with potential applications in network analysis, VLSI design and computer vision, among others. In this work, we generalize the framework of modularity maximization for clustering on hypergraphs. To this end, we introduce a hypergraph null model, analogous to the configuration model on undirected graphs, and a node-degree preserving reduction to work with this model. This is used to define a modularity function that can be maximized using the popular and fast Louvain algorithm. We additionally propose a refinement over this clustering, by reweighting cut hyperedges in an iterative fashion. The efficacy and efficiency of our methods are demonstrated on several real-world datasets.

\end{abstract}
\keywords{Hypergraph, Modularity, Clustering}
\maketitle
\section{Introduction}

The graph clustering problem involves dividing a graph into multiple sets of nodes, such that the similarity of nodes within a cluster is higher than the similarity of nodes between different clusters. While most approaches for learning clusters on graphs assume pairwise (or dyadic) relationships between entities, many entities in real world network systems engage in more complex, multi-way (super-dyadic) relations. In such systems, modeling all relations as pairwise would lead to loss in information present in the original data. The representational power of pairwise graph models is insufficient to capture such higher order information and present it for analysis or learning tasks.

\textit{Hypergraphs} provide a  natural representation for such super-dyadic relations. A hypergraph is a generalization of a graph that allows for an `edge' to connect multiple nodes. Rather than vertices and edges, the hypergraph looks like a collection of overlapping subsets of vertices, referred to as hyperedges. A hyperedge can capture a multi-way relation; for example, in a co-citation network, a hyperedge could represent a group of co-cited papers. If this were modeled as a graph, we would be able to see which papers are citing other papers, but would not see if multiple papers are being cited by the same paper. If we could visually inspect the hypergraph, we could easily see the groups that form co-citation interactions. This suggests that the hypergraph representation is not only more information-rich, but is conducive to higher order learning tasks by virtue of its structure. Indeed, research in learning on hypergraphs has been gaining recent traction \cite{AAAI1817136,AAAI1817017,AAAI1817386,Feng2018WWWPartial}.

Analogous to the graph clustering task, \textit{Hypergraph clustering} seeks to find dense connected components within a hypergraph \cite{schaeffer2007graph}. This has been the subject of much research in various communities with applications to various problems such as VLSI placement \cite{Karypis98hmetis}, image segmentation \cite{kim2011higherimagesegmentation}, de-clustering for parallel databases \cite{liu2001hypergraphdeclustering} and modeling eco-biological systems \cite{estrada2005complex}, among others.  A few previous works on hypergraph clustering \cite{leordeanu12efficient,pelillo2013gametheory,agarwal05beyondpairwise,Shashua2006tensor,Liu2010RobustCA} focused on k-uniform hypergraphs. Within the machine learning community, the authors of
 \cite{zhou2007learningscholkopf}, were among the earliest to look at learning on hypergraphs in the general case. They sought to support Spectral Clustering methods on hypergraphs and defined a suitable hypergraph Laplacian. This effort, like many other existing methods for hypergraph learning, makes use of a reduction of the hypergraph to a graph \cite{Agarwal06higherorder} and has led to much follow up work  \cite{AnandLouis2015Laplacian}.

An alternative methodology for clustering on simple graphs (those with just dyadic relations) is \textit{modularity maximization} \cite{newman2006modularity}. This class of methods, in addition to providing a useful metric for measuring cluster quality in the \textit{modularity} function, also returns the number of clusters automatically and don't require the expensive eigenvector computations typically associated with Spectral Clustering. In practice, a greedy optimization algorithm known as the Louvain method \cite{Blondel08fastunfolding} is commonly used, as it is known to be fast and scalable. These are features that would be advantageous to the hypergraph clustering problem as well, since graph reductions result in a combinatorial expansion in the number of edges. However, extending the modularity function to hypergraphs is not straightforward, as  a node-degree preserving null model would be required, analogous to the graph setting. 

Encoding the hyperedge-centric information present within the hypergraph is key to the development of an appropriate modularity-based framework for clustering. One simple option would be to reduce a hypergraph to simple graph and then employ a standard modularity-based solution. However, such an approach would lose critical information encoded within the complex super-dyadic hyperedge structure.  Additionally, when viewing the clustering problem via a minimization function, i.e., minimizing the number of cut edges, there are multiple ways to cut a hyperedge. Based on where the hyperedge is cut, the proportion and assignments of nodes on different sides of the cut will change, influencing the clustering. One would want to explicitly incorporate such information during clustering.

One way of incorporating information based on properties of hyperedges or their vertices, is to introduce hyperedge weights based on a metric or function of the data. Building on this idea, we make the following contributions in this work: 

\begin{itemize}
    \item We define a null model on hypergraphs, and prove its equivalence to the configuration model for undirected graphs. We derive a node-degree preserving graph reduction to satisfy this null model. Subsequently, we define a modularity function using the above, that can be maximized using the Louvain method.
    \item We propose an iterative hyperedge reweighting procedure that leverages information from the hypergraph structure and the balance of hyperedge cuts.
    \item We empirically evaluate the resultant algorithm, titled \textit{Iteratively Reweighted Modularity Maximization} (IRMM), on a range of real-world and synthetic datasets and demonstrate both its efficacy and efficiency over competitive baselines.
\end{itemize}

\section{Background}

\subsection{Hypergraphs}

Let $V$ be a finite set of nodes and $E$ be a collection of subsets of $V$ that are collectively exhaustive. Then $G = (V,E,w)$ is a hypergraph, with vertex set $V$ and hyperedge set $E$. Each hyperedge can be associated with a positive weight $w(e)$. While a traditional graph edge has just two nodes, a hyperedge can have multiple nodes. For a vertex $v$, we can write its degree as $d(v) = \sum_{e \in E, v \in e}w(e)$. The degree of a hyperedge $e$ is the number of nodes it contains; we can write this as $\delta(e) = |e|$.

The hypergraph incidence matrix $H$ is given by $h(v,e)=1$ if vertex $v$ is in hyperedge $e$, and $0$ otherwise. $W$ is the hyperedge weight matrix, $D_v$ is the vertex degree matrix, and $D_e$ is the edge degree matrix; all of these are diagonal matrices.

\noindent {\bf Clique Reduction:} For any hypergraph, one can find its \textit{clique reduction} \cite{hadleyengenvector} by simply replacing each hyperedge  by a clique formed from its node set. The adjacency matrix for the clique reduction of a hypergraph with incidence matrix $H$ can be written as:

$$A = HWH^T$$

The hypergraph is thus reduced to a graph.  $D_v$ may be subtracted from this matrix to zero its diagonals and remove self-loops.

\subsection{Modularity}

When clustering graphs, it is desirable to cut as few edges within a cluster as possible. Modularity \cite{newman2006modularity} is a metric of clustering quality that measures whether the number of within-cluster edges is greater than its expected value. This is defined as follows:
\begin{equation} \label{basic-modularity}
    \begin{aligned}
    Q &= \frac{1}{2m}\sum_{ij}[A_{ij} - P_{ij}]\delta(g_i,g_j) \\
    &= \frac{1}{2m}\sum_{ij}B_{ij}\delta(g_i,g_j)
    \end{aligned}
\end{equation}

Here, $B_{ij} = A_{ij} - P_{ij}$ is called the modularity matrix. $P_{ij}$ denotes the expected number of edges between node $i$ and node $j$, given by a \textit{null model}. For graphs, the \textit{configuration model} \cite{newman2010networks} is used, where edges are drawn randomly while keeping the node-degree preserved. For two nodes $i$ and $j$, with (weighted) degrees $k_i$ and $k_j$ respectively, the expected number of edges between them is hence given by:

$$P_{ij} = \frac{k_i k_j}{\sum_j k_j}$$

Since the total number of edges in a given network is fixed, maximizing the number of within-cluster edges is the same as minimizing the number of between-cluster edges. This suggests that clustering can be achieved by \textit{modularity maximization}.

\section{Hypergraph Modularity}

Analogous to the configuration model for graphs, we propose a simple but novel node-degree preserving null model for hypergraphs. Specifically we have:
\begin{equation} \label{hypergraph-null-model}
P^{hyp}_{ij} = \frac{d(i) \times d(j)}{\sum_{v \in V}d(v)}
\end{equation}

We wish to use this null model with an adjacency matrix obtained by reducing the hypergraph to a graph. However, when taking the clique reduction, the degree of a node in the corresponding graph is not the same as its degree in the original hypergraph, as verified below. 

\begin{lemma}
For the clique reduction of a hypergraph with incidence matrix $H$, the degree of a node $i$ in the reduced graph is given by: 
$$k_i = \sum_{e \in E} H(i,e)w(e) (\delta(e) - 1)$$
where $\delta(e)$ and $w(e)$ are the degree and weight of a hyperedge $e$ respectively.
\end{lemma}

\begin{proof}

The adjacency matrix of the reduced graph is given by
$$A_{clique} = HWH^{T}$$

$$(HWH^T)_{ij} = \sum_{e \in E} H(i,e)w(e)H(j,e)$$

Note that we do not have to consider self-loops, since they are not cut during the modularity maximization process. This is done by explicitly setting $A_{ii}=0$ for all $i$. Taking this into account, we can write the degree of a node $i$ in the reduced graph as: \\
\begin{align*}
    k_i &= \sum_j A_{ij} \\
    &= \sum_j \sum_{e \in E} H(i,e)w(e)H(j,e) \\
    &= \sum_{e \in E} H(i,e)w(e) \sum_{j:j \ne i} H(j,e) \\
    &= \sum_{e \in E} H(i,e)w(e) (\delta(e) - 1)
\end{align*}
\end{proof}

The result of the above theorem shows that in the clique reduction of a hypergraph, the node degree is over counted by a factor of $(\delta(e) - 1)$ for each hyperedge $e$. We can hence correct the node degree of the clique reduction by scaling each $w(e)$ down by ($\delta(e) - 1)$. This leads to the following corrected adjacency matrix,

\begin{equation} \label{hypergraph-reduction}
A^{hyp} = HW(D_e - I)^{-1}H^{T}
\end{equation}

We can now verify that the above adjacency matrix preserves the hypergraph node degree.

\begin{proposition}
For the reduction of a hypergraph given by the adjacency matrix $A = HW(D_e - I)^{-1}H^{T}$, the degree of a node $i$ in the reduced graph (denoted $k_i$) is equal to its hypergraph node degree $d(i)$.
\end{proposition}

\begin{proof}

We have,

$$(HW(D_e - I)^{-1}H^{T})_{ij} = \sum_{e \in E} \frac{H(i,e)w(e)H(j,e)}{\delta(e) - 1}$$

Note again that we do not have to consider self-loops, since they are not cut during the modularity maximization process (explicitly setting $A_{ii}=0$ for all $i$). We can rewrite the degree of a node $i$ in the reduced graph as \\

\begin{align*}
    k_i &= \sum_j A_{ij} \\
    &= \sum_{e \in E} \frac{H(i,e)w(e)}{\delta(e) - 1} \sum_{j:j \ne i} H(j,e) \\
    &= \sum_{e \in E} H(i,e)w(e) \\
    &= d(i)
\end{align*}
\end{proof}

We can use this node-degree preserving reduction, with the diagonals zeroed out, to correctly implement the null model from Eq. \ref{hypergraph-null-model}. The hypergraph modularity matrix can subsequently be written as,

$$B^{hyp}_{ij} = A^{hyp}_{ij} - P^{hyp}_{ij}$$

This new modularity matrix can be used in Eq. \ref{basic-modularity} to obtain an expression for the hypergraph modularity and can then be fed to a Louvain-style algorithm.

\begin{equation} \label{hypergraph-modularity}
    Q^{hyp} = \frac{1}{2m}\sum_{ij}B^{hyp}_{ij}\delta(g_i,g_j)  
\end{equation}

\subsection{Analysis and Connections to Random Walks}

Consider the clique reduction of the hypergraph. We can distribute the weight of each hyperedge uniformly among the edges in its associated clique. All nodes within a single hyperedge are assumed to contribute equally; a given node would receive a fraction of the weight of each hyperedge it belongs to. The number of edges each node is connected to from a hyperedge $e$ is $\delta(e) - 1$.  Hence by dividing each hyperedge weight by the number of edges in the clique, we obtain the normalized weight matrix $W(D_e-I)^{-1}$. Introducing this in the weighted clique formulation results in the proposed reduction $A = HW(D_e-I)^{-1}H^T$.

Another way of interpreting this reduction is to consider a random walk on the hypergraph in the following manner - 
\begin{itemize}
    \item take a start node $i$
    \item select a hyperedge $e$ containing $i$, proportional to its weight $w(e)$
    \item select a \textit{new} node from $e$ uniformly (there are $\delta(e) - 1$ choices)
\end{itemize}
The behaviour described above is captured by the following random walk transition model - 
\begin{align*}
    P_{ij} &= \sum_{e \in E} \frac{w(e)h(i,e)}{d(i)} \frac{h(j,e)}{\delta(e)-1}\\
    \implies P &= D_v^{-1}HW(D_e-I)^{-1}H^T
\end{align*}

By comparing the above with the random walk probability matrix for graphs ($P = D^{-1}A$) we can recover the reduction $A = HW(D_e-I)^{-1}H^T$.

\subsection{Hypergraph Louvain Method} \label{hypergraph-louvain-description}

On implementation of the modularity function defined in Eq. \ref{hypergraph-modularity}, we use the Louvain method to find clusters by maximizing the hypergraph modularity. By default, the algorithm automatically returns the number of clusters. To return a fixed number of clusters $k$, we use hierarchical agglomerative clustering as a post-processing step. For the linkage criterion, we use the average linkage.

\section{Iterative Hyperedge Reweighting}

When defining hypergraph modularity, we proposed a null model that would preserve node information. We now look at ways to improve on this initial result by leveraging hyperedge information for improving clustering in an iterative fashion.

When clustering graphs, it is desired that edges within clusters are greater in number than edges between clusters. Hence when trying to improve clustering, we look at minimizing the number of between-cluster edges that get cut. For a hypergraph, this would be done by minimizing the total volume of the hyperedge cut. In \cite{zhou2007learningscholkopf}, it was observed that for a given hyperedge $e$, the volume of the cut $\partial S$ is proportional to $|e \cap S || e \cap S^c |$, for a hypergraph whose vertex set is partitioned into two sets  $S$ and $S^c$.

The product $|e \cap S || e \cap S^c |$ can be interpreted as the number of cut sub-edges within a clique reduction. We can see that this product is maximized when the cut is balanced and there are an equal number of vertices in $S$ and $S^c$. When all vertices of $e$ go into one partition and the other partition is left empty, the product is zero. A min-cut algorithm would favour cuts that are as unbalanced as possible, as a consequence of the minimization of $|e \cap S || e \cap S^c |$.

Intuitively, if there were a larger portion of vertices in one cluster and a smaller portion in the other, it is more likely that the smaller group of vertices are actually similar to the rest and should be pulled into the larger cluster. Thus when minimizing the cut size, a hyperedge that gets cut equally between clusters is less informative than a hyperedge that gets an unbalanced cut with more vertices in one cluster. We would want hyperedges that get cut to be more balanced, and more informative hyperedges, that would have gotten unbalanced cuts, to be left uncut and pulled into the cluster. 

This can be done by increasing the weights of hyperedges that get unbalanced cuts, and decreasing the weights of hyperedges that get more balanced cuts. We know that an algorithm that tries to minimize the volume of the hyperedge boundary would try to cut as few heavily weighted hyperedges as possible. Since the hyperedges that had more unbalanced cuts get a higher weight, they are less likely to be cut after re-weighting, and instead would reside inside a cluster. Hyperedges that had more balanced cuts get a lower weight, and on re-weighting, continue to get balanced cuts. Thus after re-weighting and clustering, we would observe fewer hyperedges between clusters, and more hyperedges pushed into clusters. 

\begin{figure}[h]
\centering
\begin{minipage}{0.25\textwidth}
\centering
\includegraphics[width=0.99\linewidth]{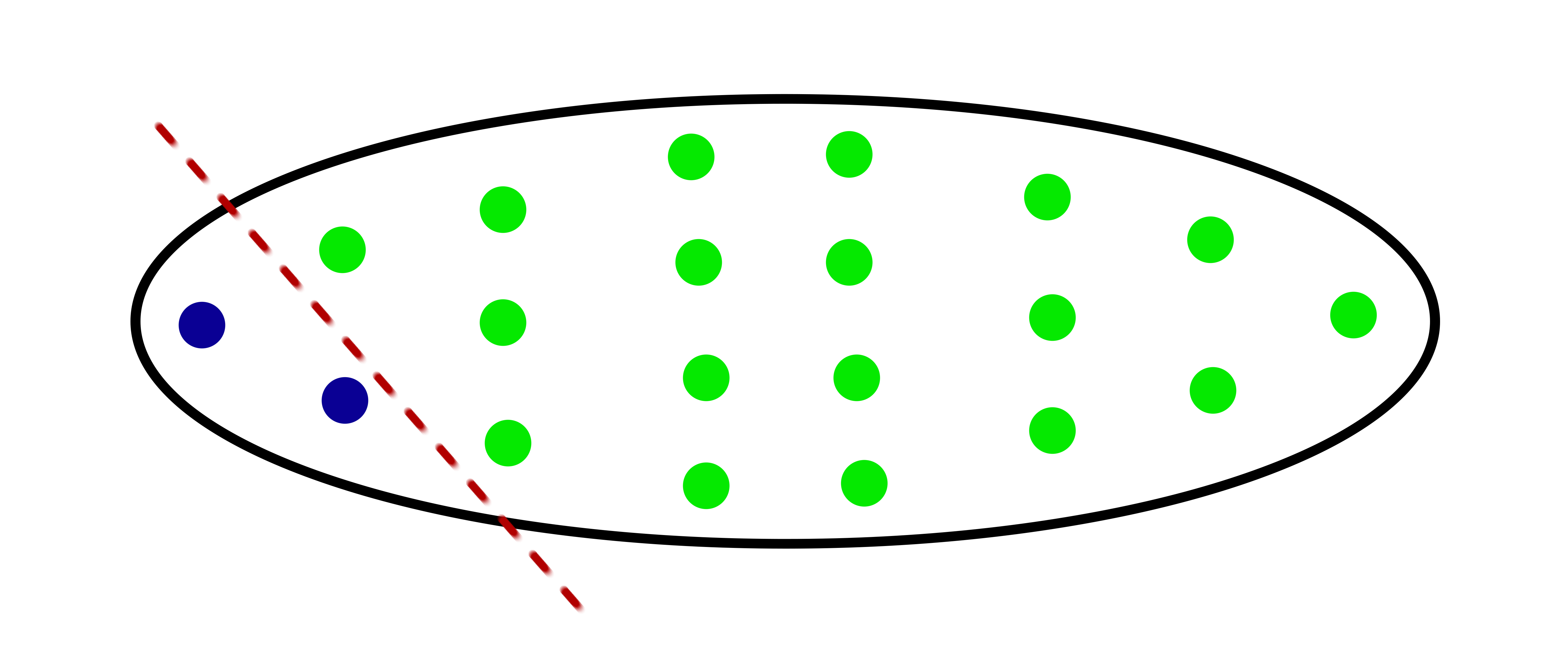}
\label{fig-cuts-reweight-before}
\begin{equation*}
    t = \bigg(\frac{1}{2} + \frac{1}{18}\bigg) \times 20 = 11.111
\end{equation*}
\end{minipage}%
\begin{minipage}{0.25\textwidth}
\centering
\includegraphics[width=0.99\linewidth]{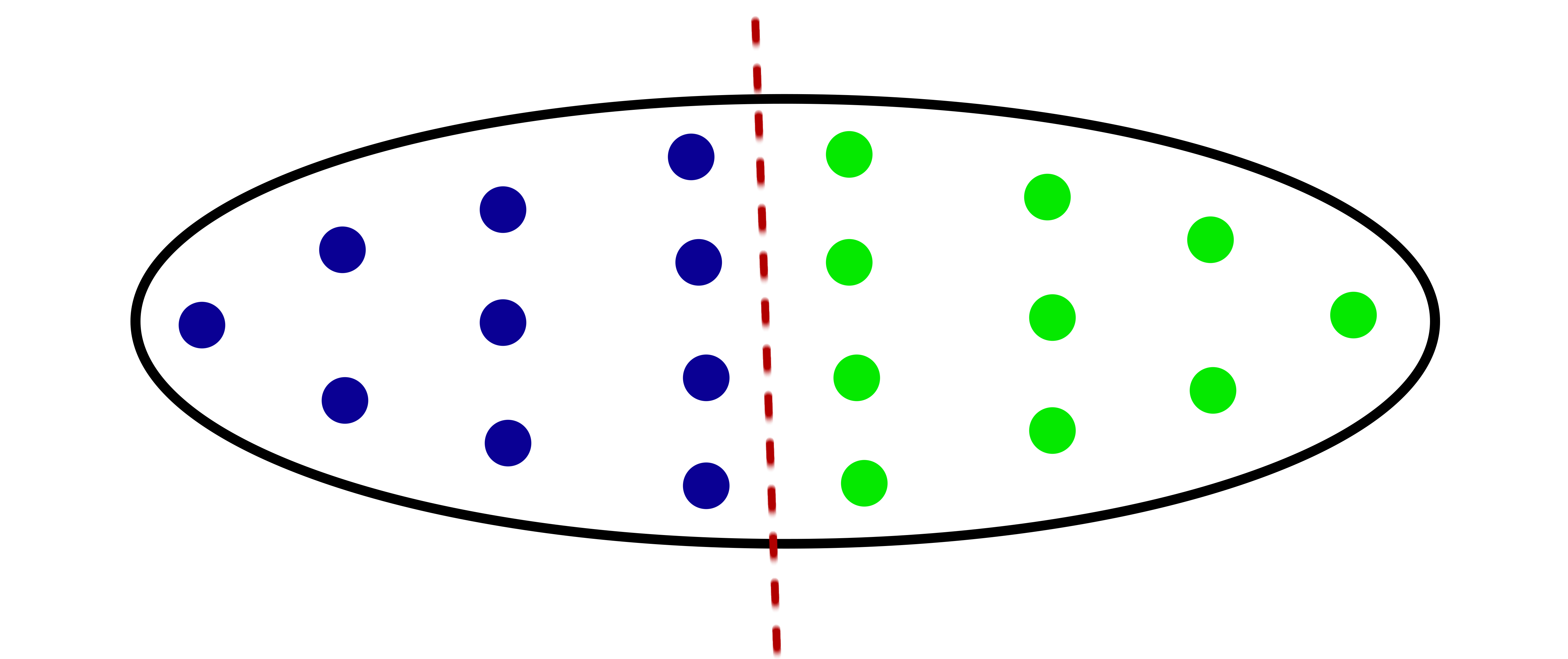}
\label{fig-cuts-reweight-after}
\begin{equation*}
    t = \bigg(\frac{1}{10} + \frac{1}{10}\bigg) \times 20 = 4
\end{equation*}
\end{minipage}
\label{fig-cuts-reweight}
\caption{Reweighting for different hyperedge cuts}
\end{figure}

Now, we formally develop a re-weighting scheme that satisfies the properties described above - increasing weight for a hyperedge that received a more unbalanced cut, and decreasing weight for a hyperedge that received a more balanced cut. Considering the case where a hyperedge is partitioned into two, for a cut hyperedge with $k_1$ and $k_2$ nodes in each partition ($k_1, k_2 \ne 0$), we have the following equation that operationalizes the aforementioned scheme - 

\begin{equation} \label{2-clusters-reweight}
t = \Big( \frac{1}{k_1} + \frac{1}{k_2} \Big) \times \delta(e)
\end{equation}

Here the multiplicative coefficient, $\delta(e)$, seeks to keep $t$ independent of the number of vertices in the hyperedges. Note that for two partitions, $\delta(e) = k_1+k_2$.

To see why this satisfies our desired property, note that $t$ is minimized when $k_1$ and $k_2$ are equal, that is, each partition gets $\delta(e)/2$ nodes and $t = 4$. For a partition with a 75:25 ratio, we would have $t = \frac{16}{3}$, and for a 95:5 ratio we would have $t = \frac{400}{19}$. It can be easily verified that $t$ is indeed minimized when $k_1=k_2=\delta(e)/2$.

We can then generalize Eq. \ref{2-clusters-reweight} to $c$ partitions as follows - 

\begin{equation} \label{c-clusters-reweight}
w^{\prime}(e) = \frac{1}{m}\sum_{i=1}^{c}\frac{1}{k_i + 1}[\delta(e) + c]
\end{equation}

Here, both the $+1$ and $+c$ terms are added for smoothing, to account for cases where any of the $k_i$'s are zero. Additionally, $m$ is the number of hyperedges, and the division by $m$ is added to normalize the weights.

Let $w_t(e)$ be the weight of hyperedge $e$ in the $t^{th}$ iteration, and $w^{\prime}(e)$ be the weight computed at a given iteration (using Eq. \ref{c-clusters-reweight}). We obtain the weight update rule by taking a moving average as follows.

\begin{equation} \label{reweight-moving-avg}
w_{t+1}(e) = \alpha w_t(e) + (1-\alpha) w^{\prime}(e)
\end{equation}

The complete algorithm for modularity maximization on hypergraphs with iterative reweighting, entitled \textit{Iteratively Reweighted Modularity Maximization (IRMM)}, is described in Algorithm 1. We are now in a position to evaluate our ideas empirically.

\begin{table*}[h]
\centering
\begin{tabular}{l c c c c c}
\hline
Dataset & No. of nodes & No. of hyperedges & Average hyperedge degree & Average node degree & No. of classes \\ \hline
TwitterFootball & 234 & 3587 & 15.491 & 237.474 & 20 \\ 
Cora & 2708 & 2222 & 3.443 & 2.825 & 7 \\
Citeseer & 3264 & 3702 & 27.988 & 31.745 & 6 \\
MovieLens & 3893 & 4677 & 79.875 & 95.961 & 2 \\
Arnetminer & 21375 & 38446 & 4.686 & 8.429 & 10 \\ \hline
\end{tabular}
\caption{Dataset Description}
\label{data-table}
\end{table*}

\begin{algorithm}[h] \label{iterative-louvain}
\SetEndCharOfAlgoLine{}
    \SetKwInOut{Input}{input}
    \SetKwInOut{Output}{output}
    \Input{Hypergraph incidence matrix $H$, vertex degree matrix $D_v$, hyperedge degree matrix $D_e$, hyperedge weights $W$ }
    \Output{Cluster assignments \textit{cluster\_ids}, number of clusters $c$}
    // Initialize weights as $W \leftarrow I$ if the hypergraph is unweighted\\
    \Repeat{$\lVert W - W_{prev} \rVert < threshold$}{
        // Compute reduced adjacency matrix \\
        $A \leftarrow HW(D_e - I)^{-1}H^T$\\
        // Zero out the diagonals of A \\
        $A \leftarrow zero\_diag(A)$ \\ 
        // Return number of clusters and cluster assignments \\
        cluster\_ids, c = LOUVAIN\_MOD\_MAX(A) \\
        // Compute new weight for each hyperedge \\
        \For{$e \in E$}{
        // Compute the number of nodes in each cluster \\
            \For {$i \in [1,..,c]$}{
                // Set of nodes in cluster i \\
                $C_i \leftarrow cluster\_assignments[i]$ \\
                $k_i = |e \cap C_i|$ \\
            }
        // Compute new weight \\
        $w^{\prime}(e) = \frac{1}{m}\sum_{i=1}^{c} \frac{1}{k_i + 1} ( \delta(e) + c)$ \\
        // Take moving average with previous weight \\
        $W_{prev}(e) \leftarrow W(e)$ \\
        $W(e) = \frac{1}{2}(w^{\prime}(e) + W_{prev}(e))$ \\
        }    
    }
    \caption{Iteratively Reweighted Modularity Maximization (IRMM)}
\end{algorithm}

\subsection{A simple example}

Figure \ref{fig-cuts-louvaintoy} illustrates the change in hyperedge cuts on a toy hypergraph for a single iteration.

\begin{figure}[h]
\centering
\begin{minipage}{0.25\textwidth}
\centering
\includegraphics[width=0.99\linewidth]{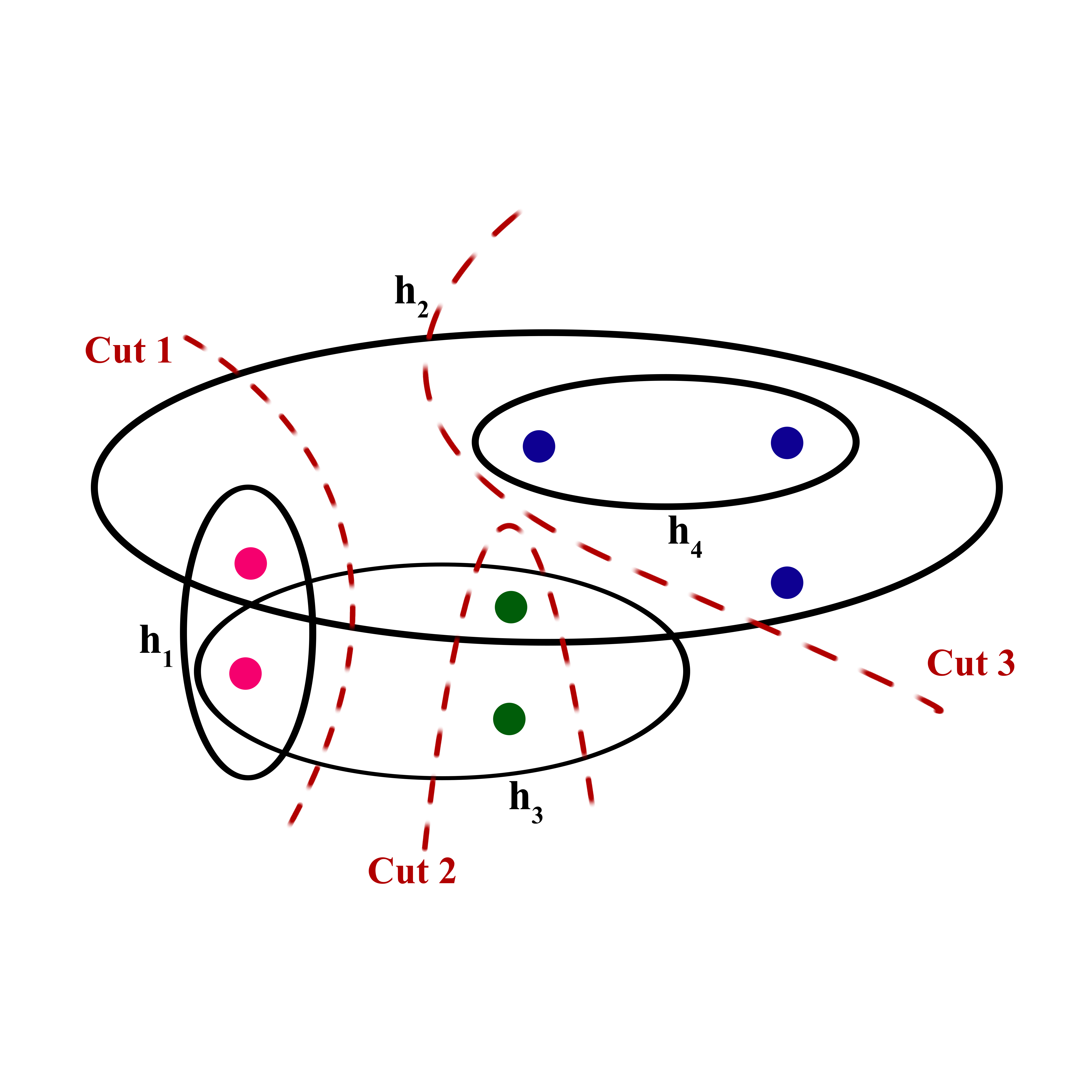}
\label{fig-cuts-louvaintoy-before}
\subcaption{Before Reweighting}
\end{minipage}%
\begin{minipage}{0.25\textwidth}
\centering
\includegraphics[width=0.99\linewidth]{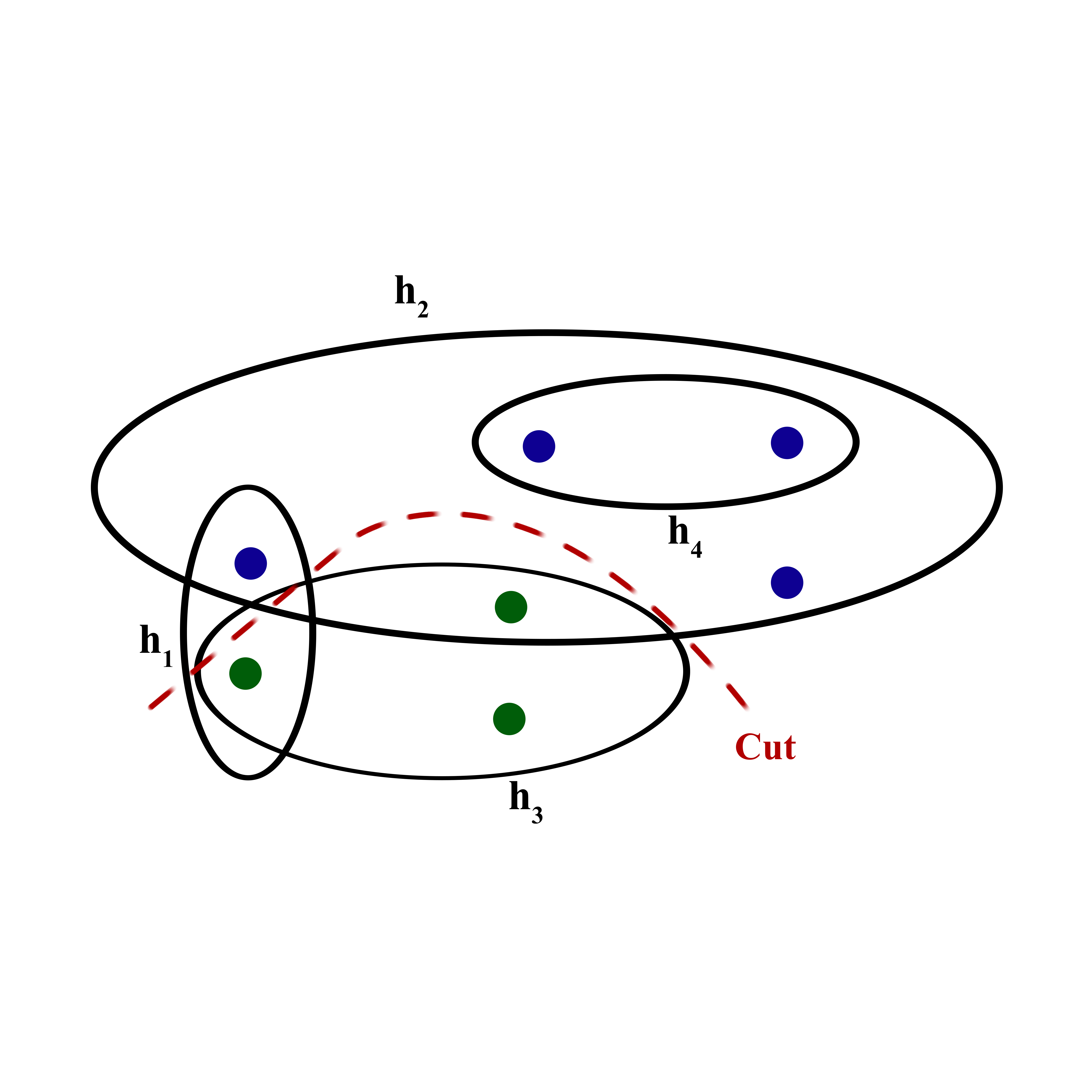}
\label{fig-cuts-louvaintoy-after}
\subcaption{After Reweighting}
\end{minipage}
\caption{Effect of iterative reweighting}
\label{fig-cuts-louvaintoy}
\end{figure}

Initially when clustering this hypergraph by modularity maximization, the hypergraph had two highly unbalanced cuts. Cut 1 and Cut 2 each split hyperedge $h_2$ in a $1:4$ ratio, and also each split hyperedge $h_3$ in a $1:2$ ratio respectively. Cut 3 splits hyperedge $h_1$ in a $2:3$ ratio. Once the reweighting is done, the $1:4$ splits are removed. This ends up reducing the number of cuts from 3 to just 1, leaving two neat clusters. On visual inspection, this cluster assignment makes sense; the single nodes in $h_1$ and $h_3$, initially assigned to another cluster, have been pulled back into their respective (larger) clusters. This captures the intended behaviour for the re-weighting scheme as described earlier.
\section{Evaluation on Ground Truth}

We used the symmetric F1 measure \cite{yang2012leskovecF1} to evaluate the clustering performance of our proposed methods on real-world data with ground-truth labels.

The proposed methods are shown in the results table as \textit{Hypergraph-Louvain} and \textit{IRMM}.

\subsection{Settings for IRMM}

We tried tuning the hyperparameter $\alpha$ using a grid search over the set of values ${0.1,0.2,...,0.9}$. We did not find a difference in the resultant F1 scores, and minimal difference in the rate of convergence, over a wide range of values (for example, $0.3$ to $0.9$ on the TwitterFootball dataset). Being a scalar coefficient in a moving average, it did not result in a large difference in resultant weight values when set in this useful range. Hence, for the final experiments, we chose to leave it at $\alpha = 0.5$. For the iterations, we set the stopping threshold at $0.01$.

\subsection{Compared Methods} 
To evaluate the performance of our proposed methods, we compared against the following baselines.

\textbf{Clique Reductions:} We took the clique reduction of the hypergraph ($A = HWH^T$) and ran the graph versions of Spectral Clustering and the Louvain method. These are referred to in the table as \textit{Clique-Spectral} and \textit{Clique-Louvain} respectively.

\textbf{Hypergraph Spectral Clustering:} We use the hypergraph Laplacian as defined in \cite{zhou2007learningscholkopf}. The top $k$ eigenvectors of the Laplacian are found, and then clustered using bisecting k-means. In the results table, this method is referred to as \textit{Hypergraph-Spectral}.

\textbf{hMETIS\footnote{http://glaros.dtc.umn.edu/gkhome/metis/hmetis/download} and PaToH\footnote{http://bmi.osu.edu/umit/software.html}:} These are hypergraph partitioning algorithms that are commonly used in practice. We used the original implementations provided by the respective authors.

\subsection{Datasets}

For all datasets, we took the single largest connected component of the hypergraph. The class labels were taken as ground truth clusters. Table 1 gives the dataset statistics.

\textbf{MovieLens \footnote{http://ir.ii.uam.es/hetrec2011/datasets.html}:} This is a multi-relational dataset provided by GroupLens research, where movies are represented by nodes. We used the \textit{director} relation to define hyperedges and build a co-director hypergraph. A group of nodes are connected by a hyperedge if they were directed by the same individual.

\textbf{Cora and Citeseer}: These are bibliographic datasets, where the nodes represent papers. In each, the nodes are connected by a hyperedge if they share the same set of words (after removing low frequency  and stop words). \cite{sen2008collective}.

\textbf{TwitterFootball:} This is taken from one view of the Twitter dataset \cite{Greene2013AggregatingContent}, and it represents members of 20 different football clubs of English Premier League. Here, the nodes are the players, and hyperedges are formed based on whether they are co-listed.

\begin{table*}
\centering
\begin{tabular}{l c c c c c c}
\hline
 & {\color[HTML]{000000} Citeseer} &
 {\color[HTML]{000000} Cora} &
 {\color[HTML]{000000} MovieLens} &
 {\color[HTML]{000000} TwitterFootball} &
 {\color[HTML]{000000} Arnetminer}\\\hline
 hMETIS & 0.2278 & 0.3787 & 0.4052 & 0.1648 & 0.0264 \\
  PaToH & 0.2095 & 0.1890 & 0.4008 & 0.0797 & 0.0023 \\ 
Clique-Spectral & 0.3408 & 0.2822 & 0.2052 & 0.1974 & 0.0220 \\
Hypergraph-Spectral & 0.3337 & 0.3051 & 0.2080 & 0.3596 & 0.0292 \\
Clique-Louvain & 0.3521 & 0.3027 & 0.3254 & 0.1412 & 0.0097 \\
Hypergraph-Louvain & 0.3673 & 0.4637 & 0.3498 & 0.3812 & 0.0686  \\
IRMM & \textbf{0.4887} & \textbf{0.5342} & \textbf{0.4134} & \textbf{0.5463} & \textbf{0.0754} \\ \hline
\end{tabular}
\caption{Symmetric F1 scores against ground truth, with the number of clusters returned by the Louvain method}
\label{F1table-louvain_k}
\end{table*}

\begin{table*}
\centering
\begin{tabular}{l c c c c c c}
\hline
 & {\color[HTML]{000000} Citeseer} &
 {\color[HTML]{000000} Cora} &
 {\color[HTML]{000000} MovieLens} &
 {\color[HTML]{000000} TwitterFootball} &
 {\color[HTML]{000000} Arnetminer}\\\hline
 hMETIS & 0.1941 & 0.3876 & 0.5116 & 0.3023 & 0.2729 \\
  PaToH & 0.1734 & 0.1588 & 0.5015 & 0.0595 & 0.1551 \\
Clique-Spectral & 0.3229 & 0.2735 & 0.508 & 0.4941 & 0.2134 \\
Hypergraph-Spectral & 0.3721 & 0.2860 & 0.5217 & 0.5059 & 0.3555 \\
Clique-Louvain & 0.3521 & 0.3027 & 0.5253 & 0.1412 & 0.4139 \\ 
Hypergraph-Louvain & 0.3010 & 0.3628 & 0.6685 & 0.3812 & 0.5140  \\
IRMM & \textbf{0.4270} & \textbf{0.3885} & \textbf{0.6693} & \textbf{0.5463} & \textbf{0.5154} \\ \hline
\end{tabular}
\caption{Symmetric F1 scores against ground truth, with the number of clusters set to the number of ground truth classes}
\label{F1table_groundtruth_k}
\end{table*}

\begin{table*}
\centering
\begin{tabular}{l c c c c c c}
\hline
 & Citeseer & Cora & MovieLens & TwitterFootball & Arnetminer \\ \hline
Hyper-Spectral & 84.16 & 41.44 & 155.8 & 3.88 & 34790 \\ 
Hyper-Louvain & 41.21 & 24.23 & 35.9 & 3.32 & 4311.2 \\ \hline
\end{tabular}
\caption{CPU times (in seconds) for the hypergraph clustering methods on all datasets}
\label{dataset-runtime-table}
\end{table*}

\textbf{Arnetminer:} This is also a co-citation network, but with a significantly larger number of nodes. Here, the nodes are papers, and they are connected by hyperedges if they are co-cited. We used the nodes from the Computer Science discipline, and its 10 sub-disciplines were treated as clusters.

\subsection{Experiments}

We compare the F1 scores for the different datasets on all the given methods. The number of clusters was first set to that returned by the Louvain method, in an unsupervised fashion. This is what would be expected in a real-world setting, where the number of clusters is not given apriori. Table 2 shows the results of this experiment. 

Secondly, we ran the same experiments with the number of clusters set to the number of ground truth classes, using the postprocessing step described in Section \ref{hypergraph-louvain-description}. The results of this experiment are given in Table 3.

We also plotted the results for varying number of clusters using the same methodology described above, to assess our method's robustness. The results are shown in Fig. 3.

\subsection{Results}

Firstly, we note that in both experiment settings, \textit{IRMM} shows the best performance on all but one of the datasets. In particular, it showed an improvement over \textit{Hypergraph-Louvain} on all datasets.

Additionally, it is evident that \textit{Hypergraph-Spectral} and \textit{Hypergraph-Louvain} consistently outperform the respective clique reduction based methods (\textit{Clique-Spectral} and \textit{Clique-Louvain}) on all datasets and both experiment settings. We infer that super-dyadic relational information captured by the hypergraph has a positive impact on the clustering performance.

Across the board, the hypergraph modularity maximization methods (\textit{Hypergraph-Louvain} and \textit{IRMM}) show competitive performance when compared to the baselines. On the MovieLens and ArnetMiner datasets, the clustering performance improves with the number of clusters set to the number of ground truth classes. We investigate this further in Section 6.

\begin{figure*}[t!p] \label{kplot}
\centering
         \begin{subfigure}[b]{0.48\textwidth}
                 \centering
                 \includegraphics[width=0.98\textwidth]{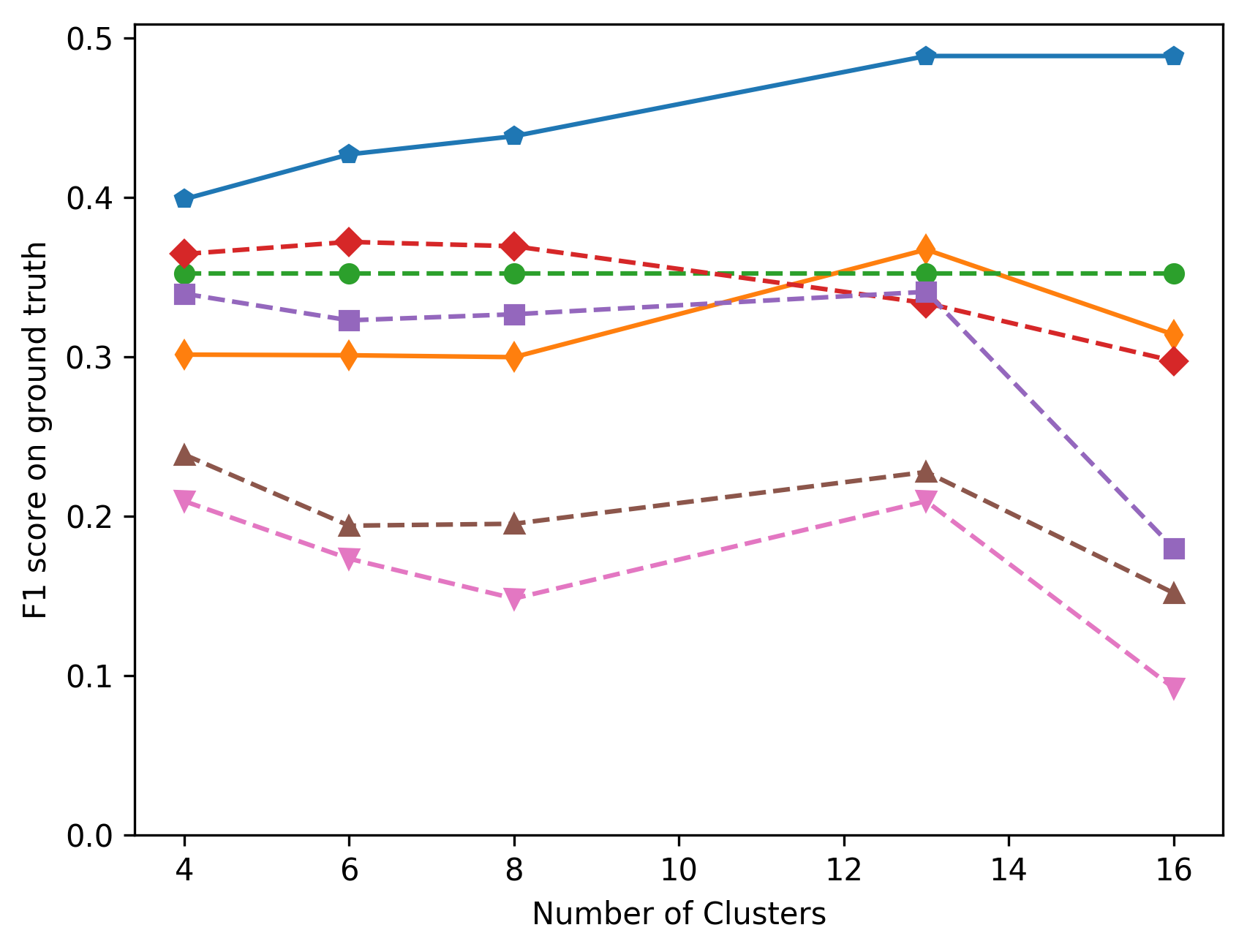}
                 \caption{Citeseer}
                 \label{citeseer:kplot}
         \end{subfigure}
         \begin{subfigure}[b]{0.48\textwidth}
                 \centering
                 \includegraphics[width=0.98\textwidth]{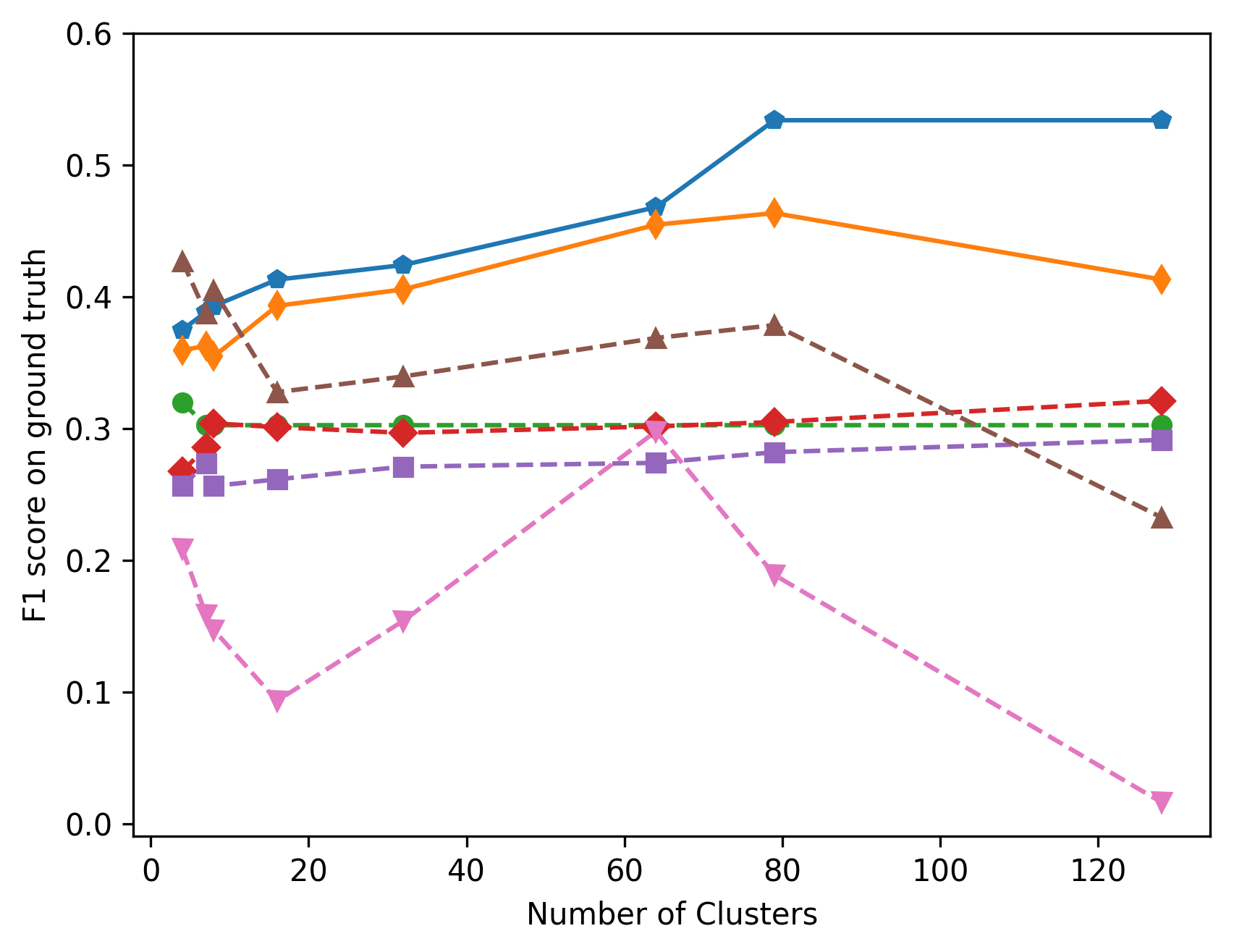}
                 \caption{Cora}
                 \label{cora:kplot}
         \end{subfigure}

         \begin{subfigure}[b]{0.48\textwidth}
                 \centering                \includegraphics[width=0.98\textwidth]{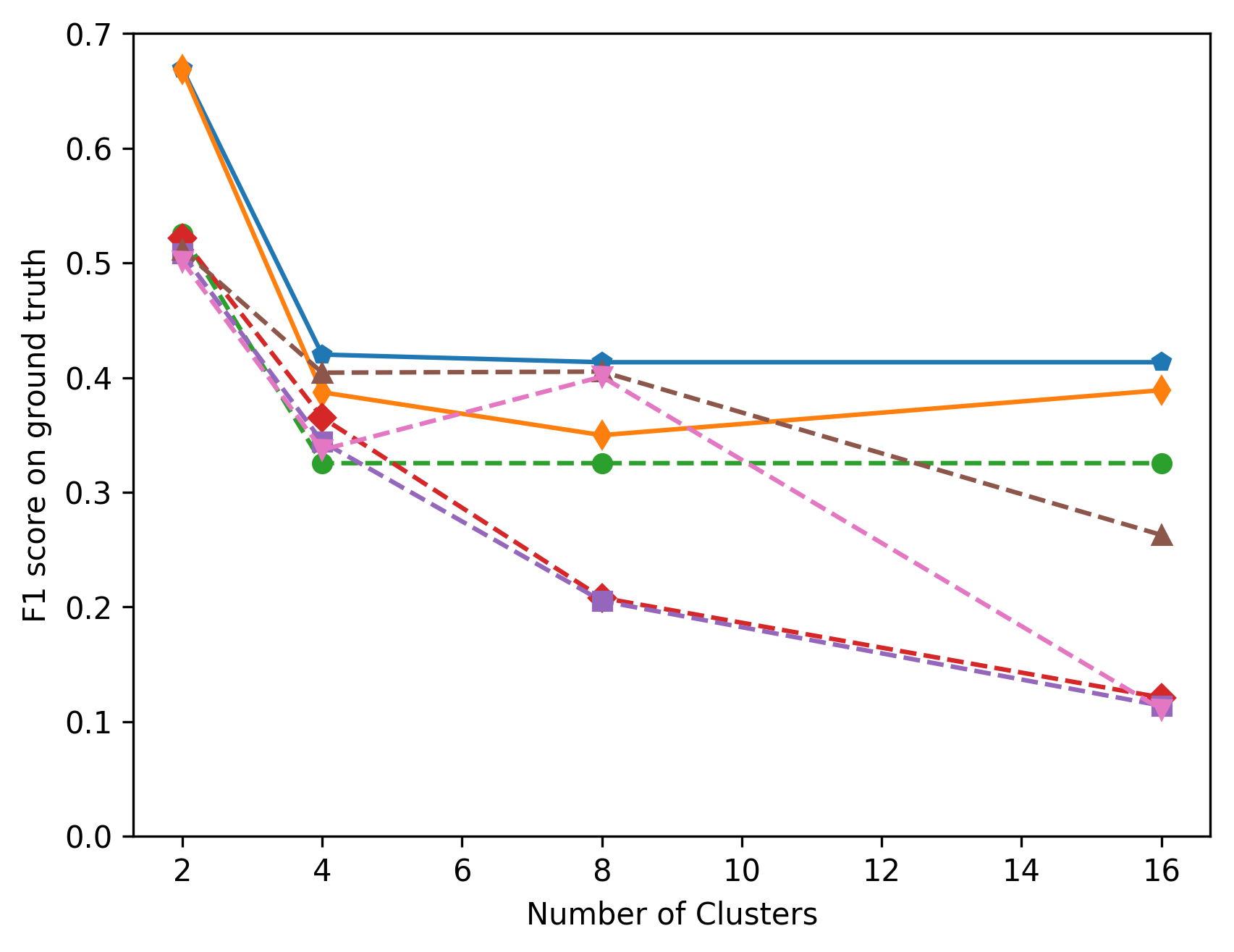}
                 \caption{MovieLens}
                 \label{movielens:kplot}
         \end{subfigure}
         \begin{subfigure}[b]{0.48\textwidth}
                 \centering                 \includegraphics[width=0.98\textwidth]{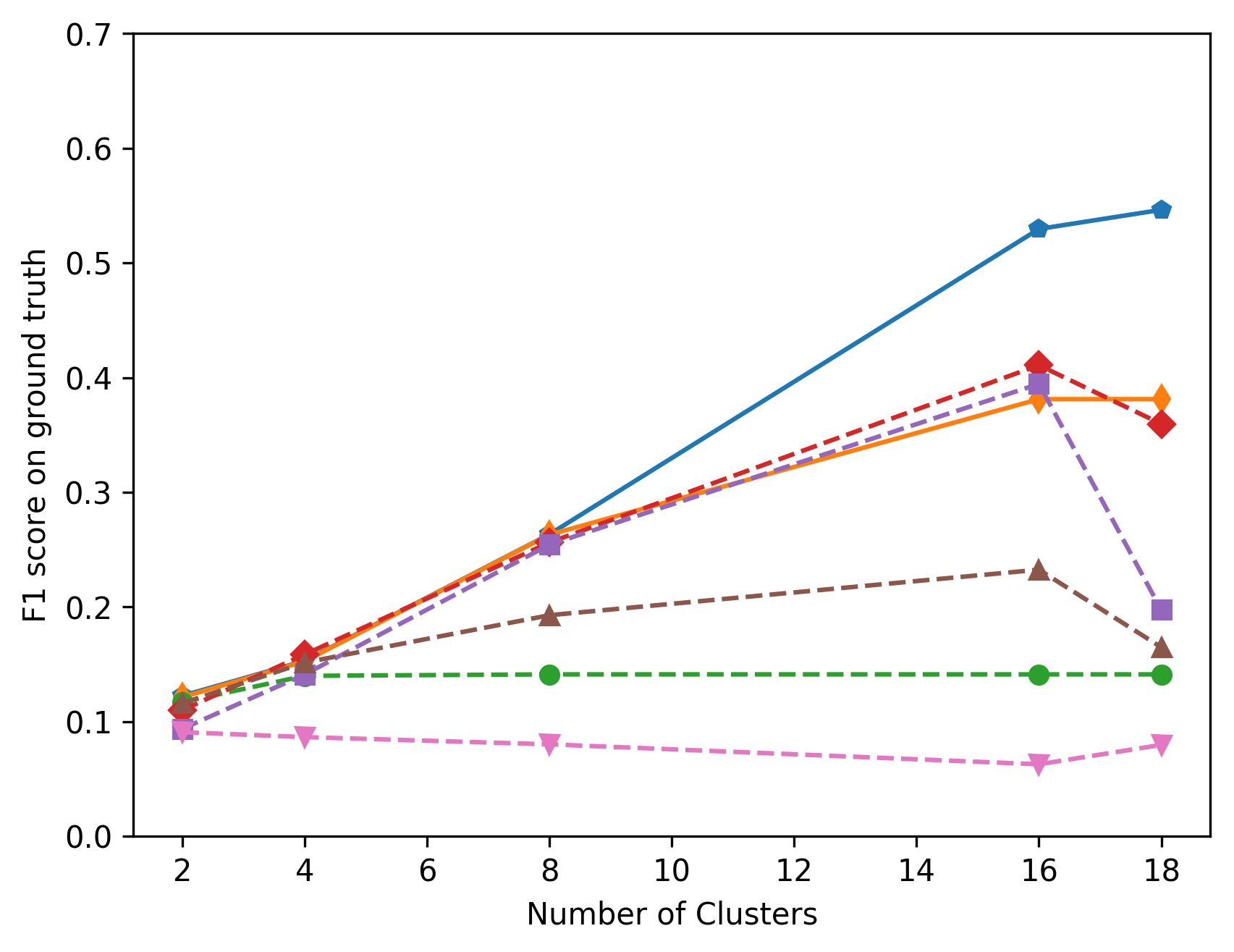}
                 \caption{TwitterFootball}
                 \label{football:kplot}
         \end{subfigure}

         \begin{subfigure}[b]{0.48\textwidth}
                 \centering                \includegraphics[width=0.98\textwidth]{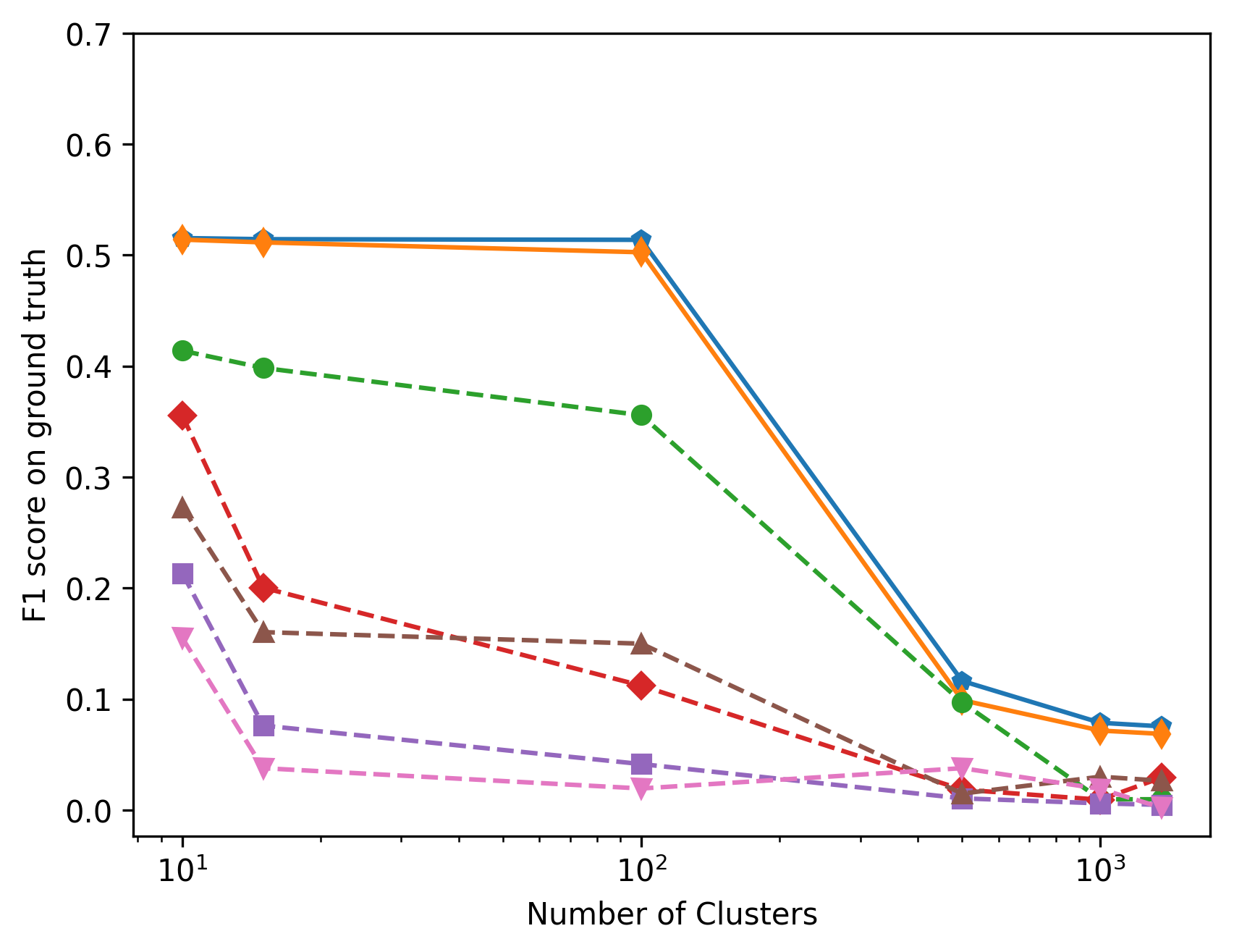}
                 \caption{Arnetminer}
                 \label{arnet:kplot}
         \end{subfigure}
         \begin{subfigure}[b]{0.48\textwidth}
                 \centering                 \includegraphics[width=0.7\textwidth]{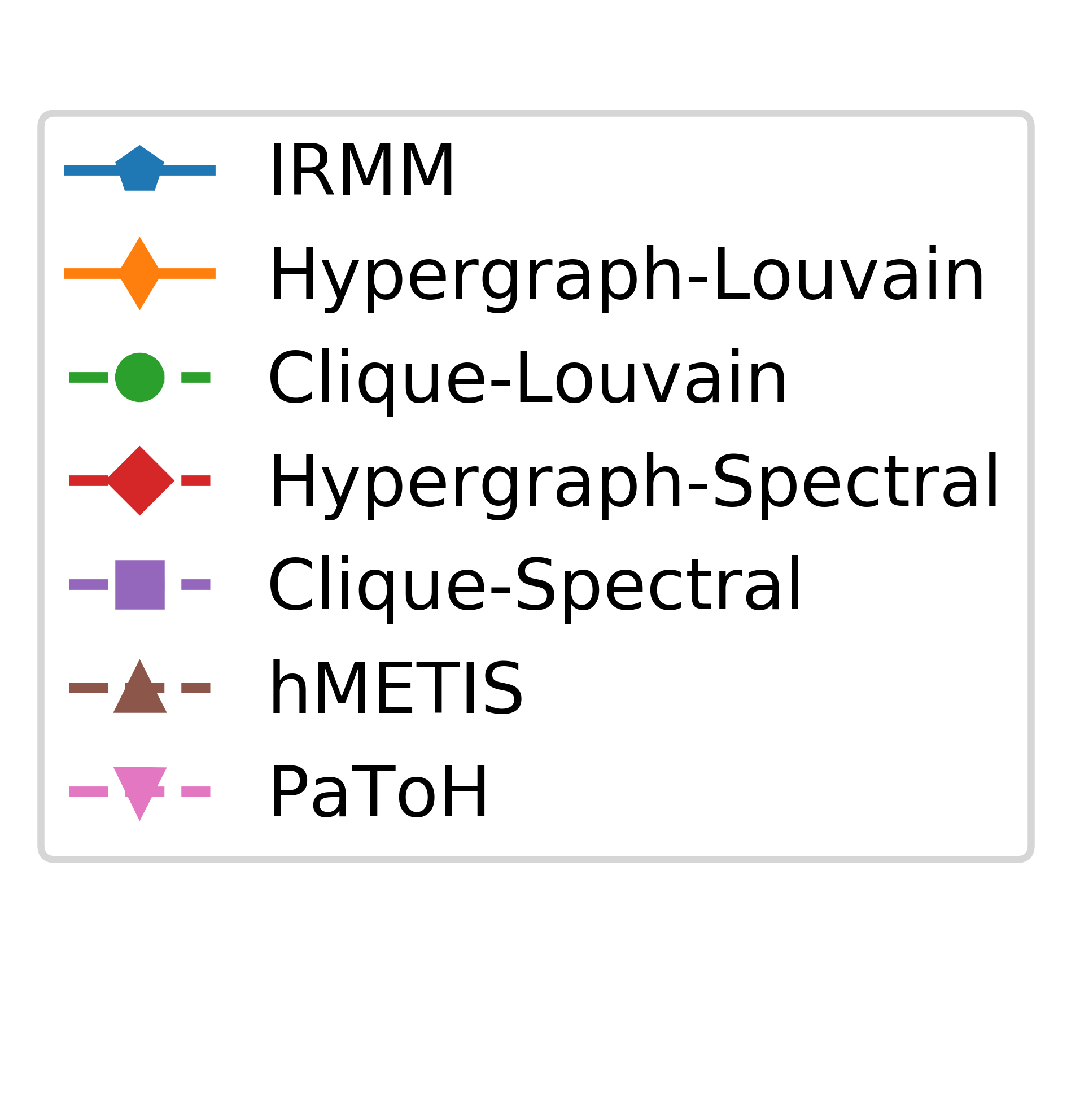}
                 \label{legend:kplot}
                 \caption*{}
         \end{subfigure}
\caption{F1 scores for varying number of clusters}
\end{figure*}

As shown in Fig. 3, the IRMM method outperforms other baseline methods on a majority of the experiment settings. This indicates the improvement in hyperedge cut due to the reweighting method. We note that on some datasets, the best performance is achieved when the number of clusters returned by the Louvain method is used (e.g Citeseer, Cora), and on others when the ground truth number of classes is used (e.g ArnetMiner, MovieLens). This could be based on the structure of clusters in the network and its relationship to the ground truth classes. A class could comprise multiple smaller clusters, which would be better detected by the Louvain algorithm. In other cases, the clustering structure could correspond better to the class label assignments. For example, the ArnetMiner network, where the class labels correspond to subdisciplines, shows much better performance when the number of clusters is brought down to the ground truth number of classes.
\section{Analysis}

In this section, we analyze the properties of the hypergraph modularity based methods. We look at the scalability of hypergraph modularity maximization. This is an important consideration for hypergraphs; for a given number of nodes $n$, there are $2^n$ possible hyperedges that can be formed, and even taking an adjacency matrix reduction of a hypergraph results in a combinatorial expansion in the number of edges. We also analyze the effect of reweighting, by looking at the proportions of each hyperedge in different clusters.

\subsection{Effect of reweighting on hyperedge cuts}

Consider a hyperedge that is cut, its nodes partitioned into different clusters. Looking at Eq. \ref{c-clusters-reweight}, we can see that $w^{\prime}(e)$ is minimized when all the partitions are of equal size, and maximized when one of the partitions is much larger than the other. The iterative reweighting procedure is designed to increase the number of hyperedges with balanced partitioning, and decrease the number of hyperedges with unbalanced partitioning. As iterations pass, hyperedges that are more unbalanced should be pushed into neighbouring clusters, and the hyperedges that lie between clusters should be more balanced.

The plots in Figure 5 illustrate the effect of hyperedge reweighting over iterations. We found the relative size of the largest partition of each hyperedge, and binned them in intervals of relative size = 0.1. The plot shows the fraction of hyperedges that fall in each bin over each iteration.

$$\text{relative size}(e) = \max_i \frac{\text{number of nodes in cluster }i}{\text{number of nodes in the hyperedge }e}$$

Here, we refer to the hyperedges as \textit{fragmented} if the relative size of its largest partition is low, and \textit{dominated} of the relative size of its largest partition is high. The fragmented edges are likely to be balanced, since the largest cluster size is low. 

On the Arnetminer dataset, we find that the number of dominated edges is already very high. Looking back at the evaluation on ground truth on this dataset, we find that the Hypergraph-Louvain method already well outperforms the baselines, and the magnitude of improvement added by the iterative method is relatively low compared to the other datasets. This is reflected in the corresponding plot in Fig. 5. The MovieLens dataset also exhibits a large number of dominated edges at the start of the algorithm.

On the smaller TwitterFootball dataset, which has a greater number of ground truth classes, we see that the number of dominated edges decreases and the number of fragmented edges increases. This is as expected; the increase in fragmented edges is likely to correspond to more balanced cuts. A similar trend is reflected in the larger Cora dataset.

The Citeseer dataset showed the largest improvement in F1 score, as shown in Tables 2 and 3. The increase in dominated edges indicates that more fragmented edges were pushed into larger clusters, becoming dominated. This corresponds to the intuition that more informative hyperedges, which are likely to be dominated, should lie inside clusters rather than between clusters. The change in cuts illustrated in the plot reflects the greater improvement in clustering quality shown in the preceding experiments.
\subsection{Scalability of the Hypergraph Louvain method}

To further motivate the extension of modularity maximization methods to the hypergraph clustering problem, we look at the scalability of the \textit{Hypergraph-Louvain} method against the strongest baseline, \textit{Hypergraph-Spectral}.

Table 3 shows the CPU times for the \textit{Hypergraph-Louvain} and \textit{Hypergraph-Spectral} on the real-world datasets. We see that while the difference is less pronounced on a smaller dataset like (\textit{TwitterFootball}, it is much greater on the larger datasets. In particular, the runtime on Arnetminer for \textit{Hypergraph-Louvain} is lower by a significant margin, not having to compute an expensive eigendecomposition.

\textbf{Analysis on synthetic hypergraphs: }On the real-world data, modularity maximization showed improved scalability as the dataset size increased. To evaluate this trend, we compared the CPU times for the \textit{Hypergraph-Spectral} and \textit{Hypergraph-Louvain} methods on synthetic hypergraphs of different sizes. For each hypergraph, we first ran \textit{Hypergraph-Louvain} and found the number of clusters returned, then ran the \textit{Hypergraph-Spectral} method with the same number of clusters. 

Following the method used in EDRW\footnote{https://github.com/HariniA/EDRW}\cite{Satchidanand2015ExtendedDR}, we generated hypergraphs with 2 classes and a homophily of 0.4 (40\% of the hyperedges deviate from the expected class distribution). The hypergraph followed a modified power-law distribution, where 75\% of its hyperedges contained less than 3\% of the nodes, 20\% of its hyperedges contained 3\%-50\% of the nodes, and the remaining 5\% contained over half the nodes in the dataset. To generate a hypergraph, we first set the number of hyperedges to 1.5 times the number of nodes. For each hyperedge, we sampled its size $k$ from the modified power-law distribution and chose $k$ different nodes based on the homophily of the hypergraph. We generated hypergraphs of sizes ranging from 1000 nodes up to 10000 nodes, at intervals of 500 nodes.

Figure \ref{fig-runtime} shows how the CPU time varies with the number of nodes, on the synthetic hypergraphs generated as given above.

\begin{figure}[h] 
\centering
\includegraphics[width=0.90\linewidth]{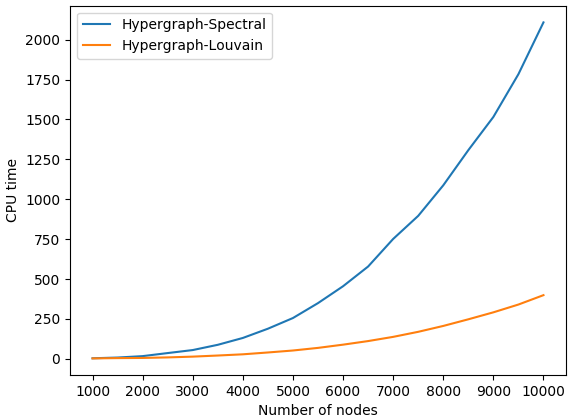}
\caption{CPU time on synthetic hypergraphs}
\label{fig-runtime}
\end{figure}

\begin{figure*}[t!p] \label{binningplots}
\centering
         \begin{subfigure}[b]{0.48\textwidth}
                 \centering
                 \includegraphics[width=0.98\textwidth]{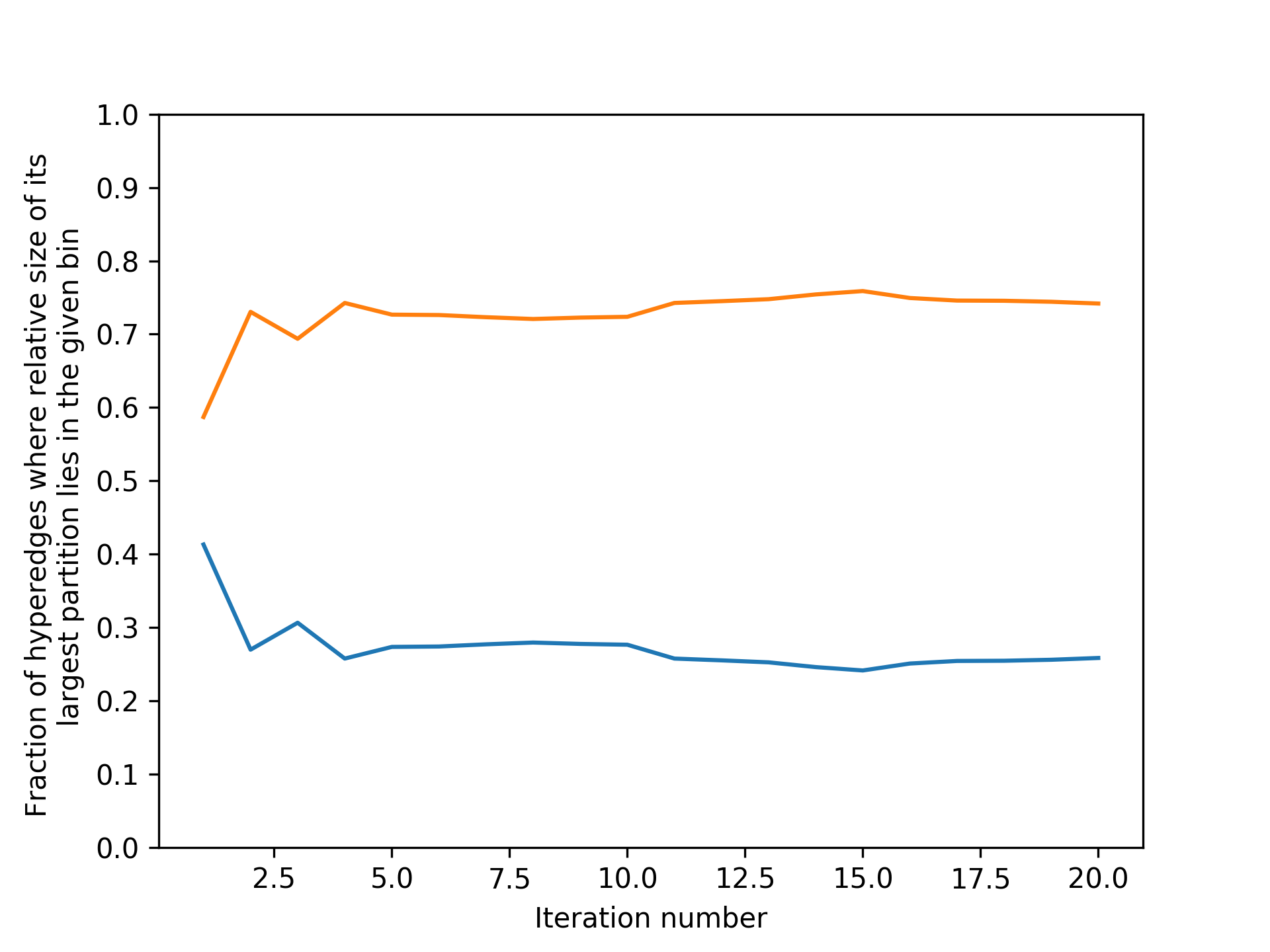}
                 \caption{Citeseer}
                 \label{citeseer:bins}
         \end{subfigure}
         \begin{subfigure}[b]{0.48\textwidth}
                 \centering
                 \includegraphics[width=0.98\textwidth]{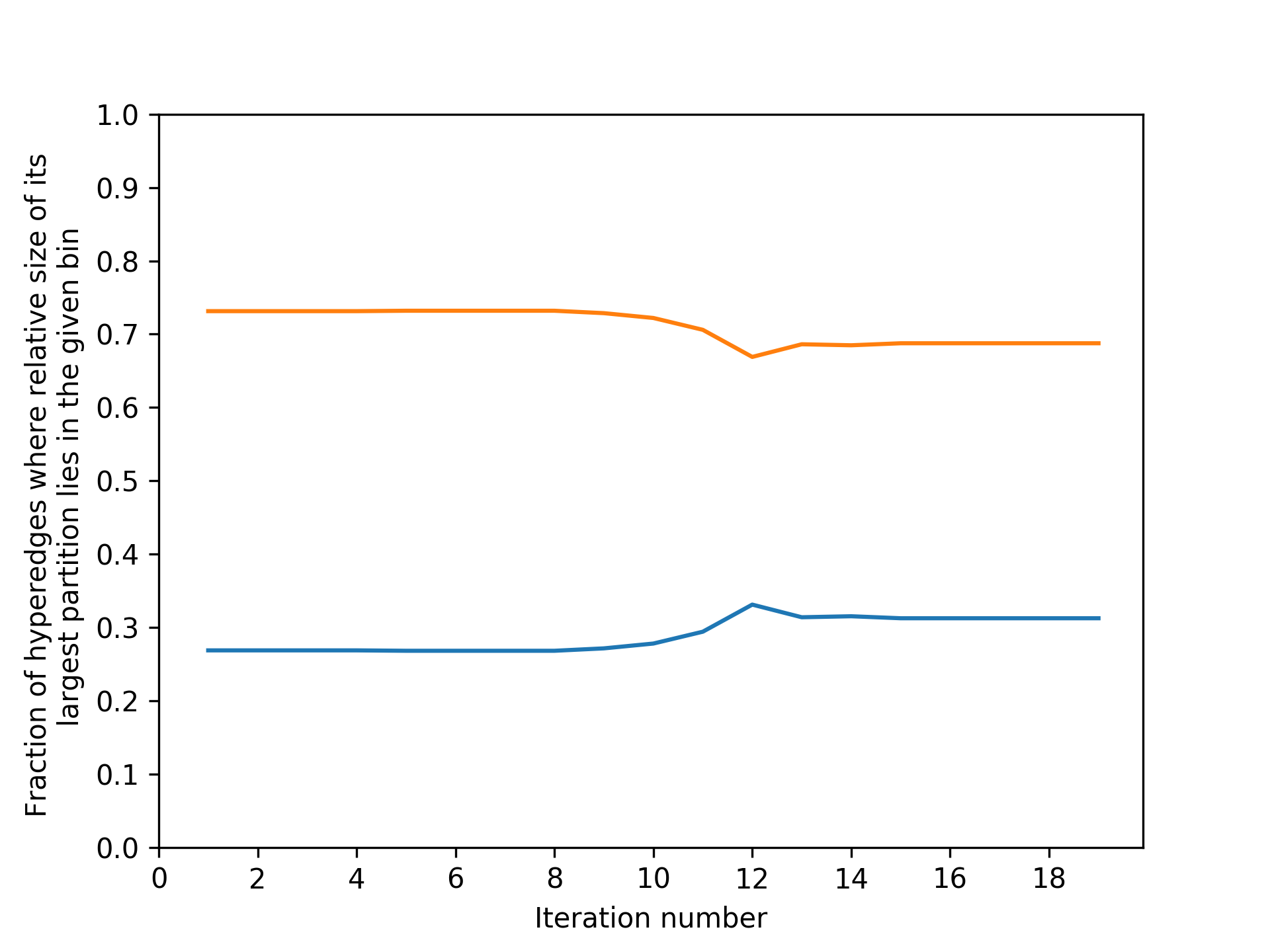}
                 \caption{Cora}
                 \label{cora:bins}
         \end{subfigure}

         \begin{subfigure}[b]{0.48\textwidth}
                 \centering
                 \includegraphics[width=0.98\textwidth]{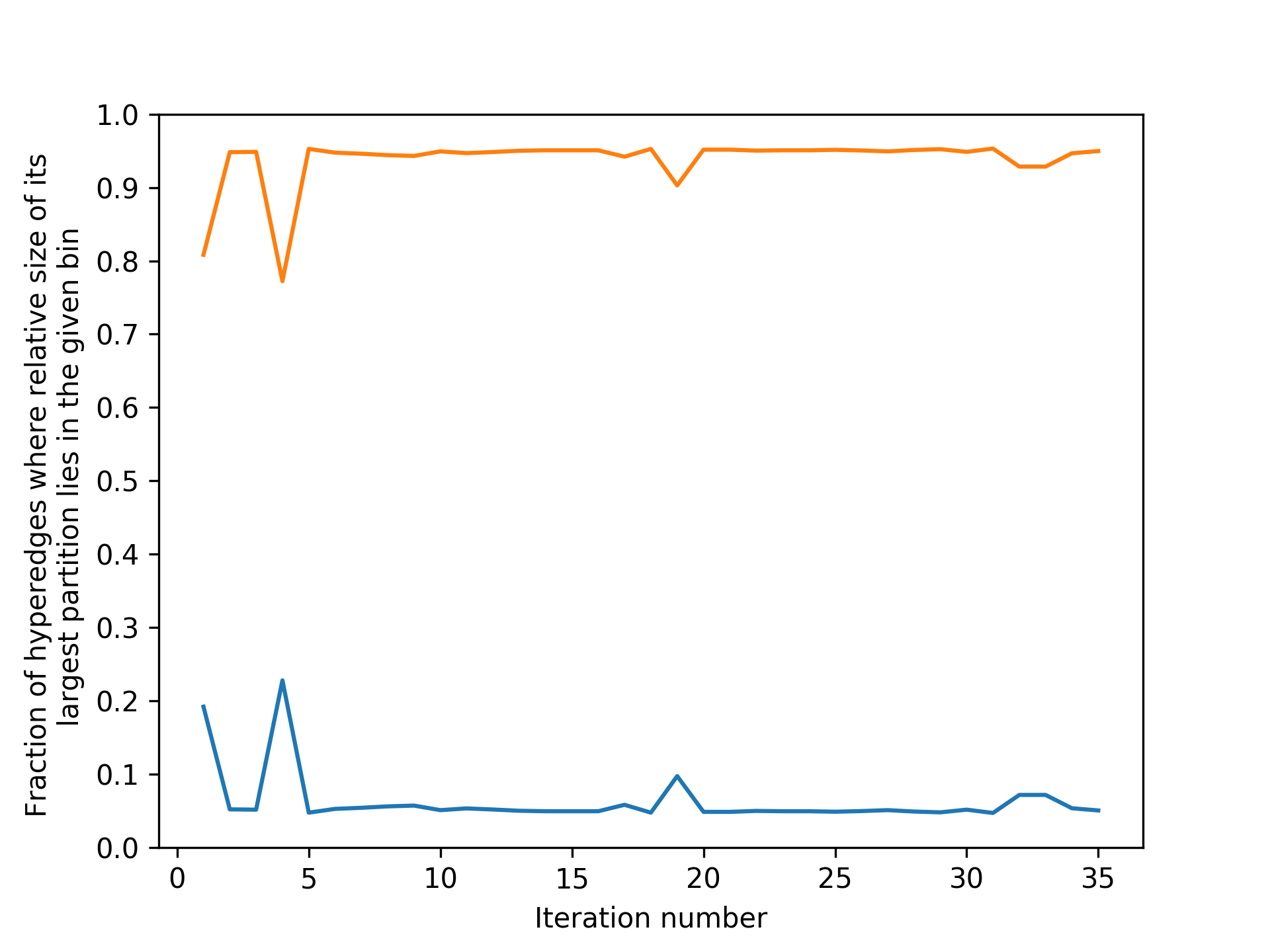}
                 \caption{MovieLens}
                 \label{movielens:bins}
         \end{subfigure}
         \begin{subfigure}[b]{0.48\textwidth}
                 \centering
                 \includegraphics[width=0.98\textwidth]{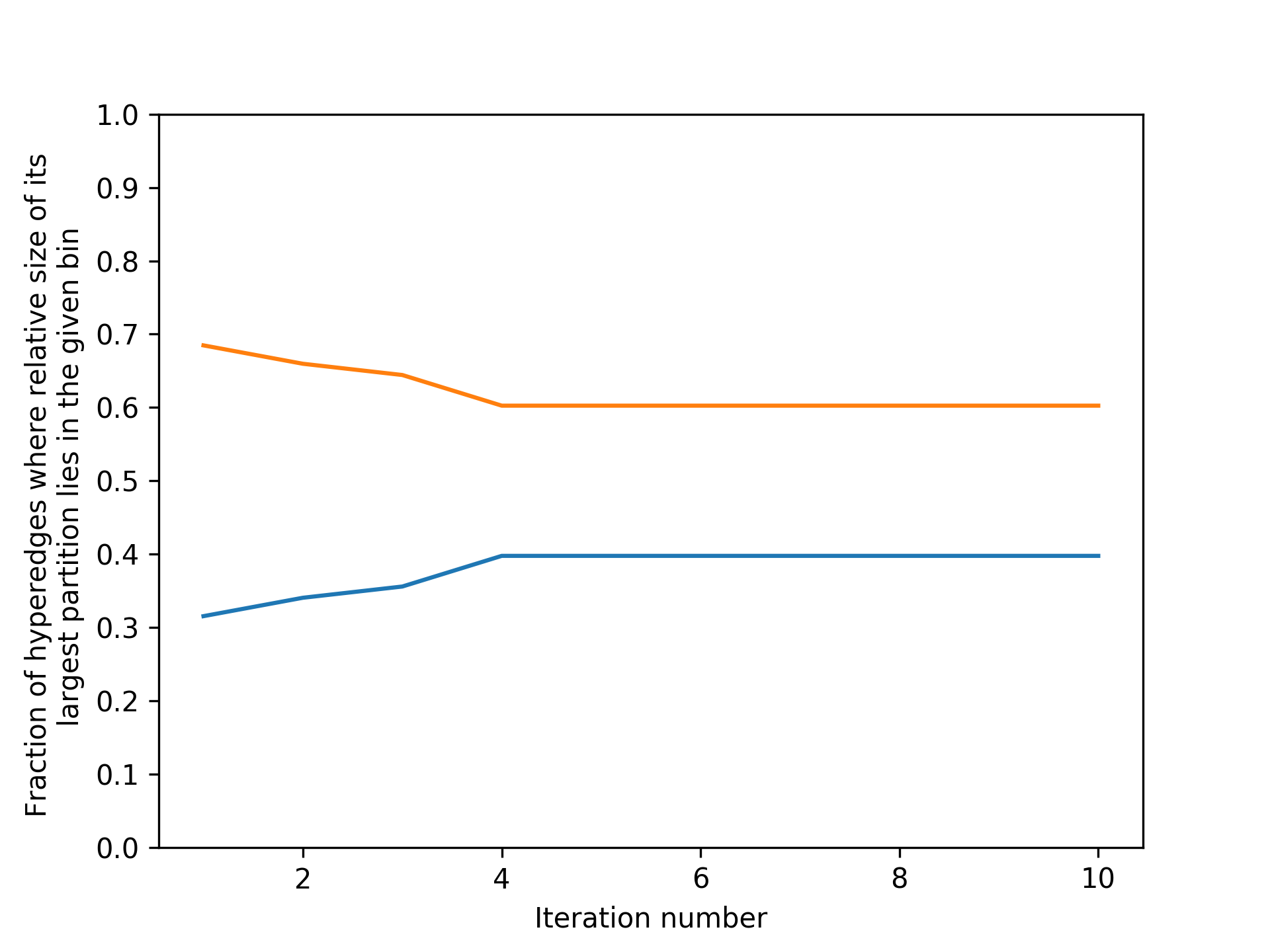}
                 \caption{TwitterFootball}
                 \label{football:bins}
         \end{subfigure}

         \begin{subfigure}[b]{0.48\textwidth}
                 \centering                \includegraphics[width=0.98\textwidth]{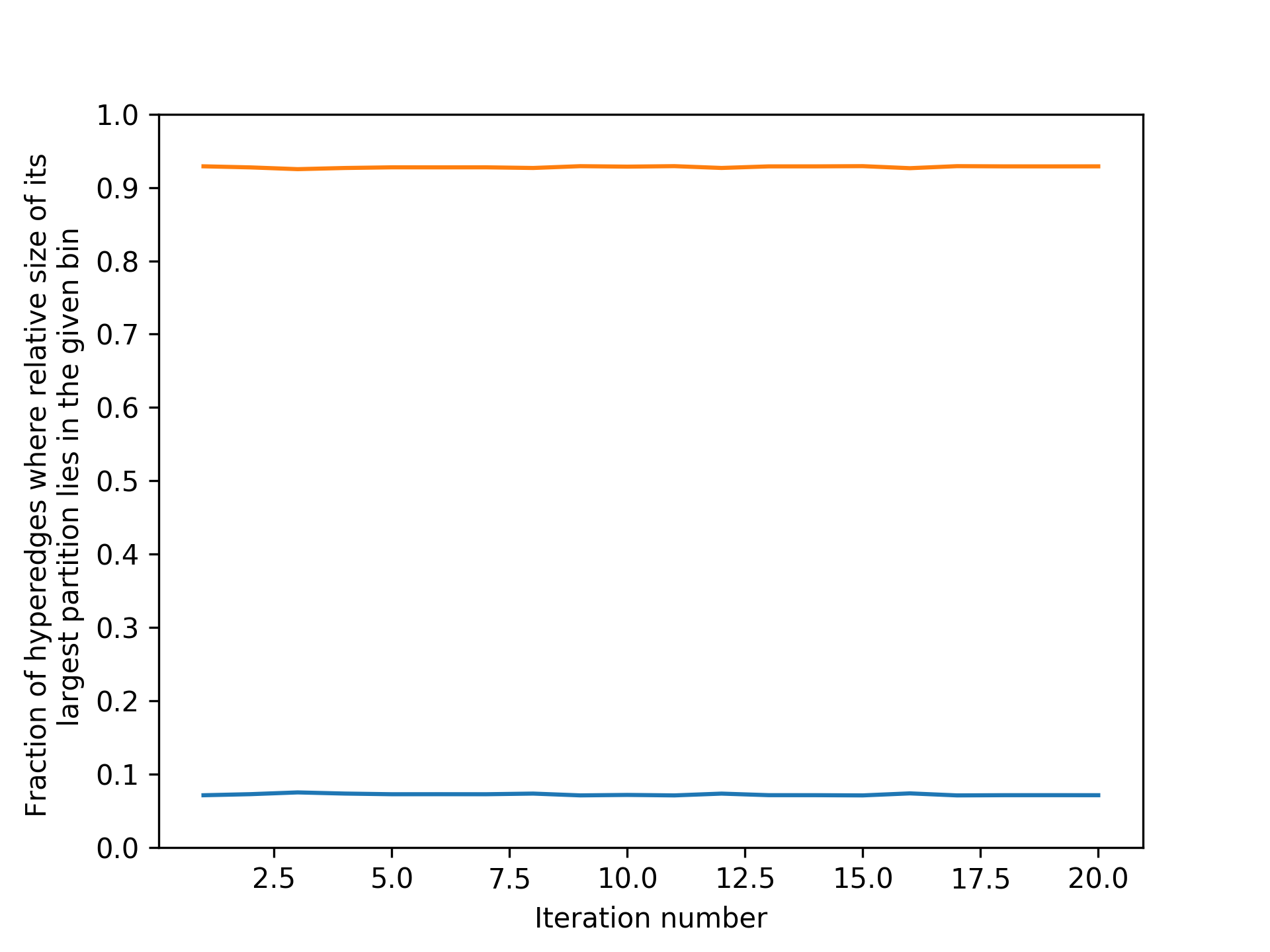}
                 \caption{Arnetminer}
                 \label{arnet:bins}
         \end{subfigure}
         \begin{subfigure}[b]{0.48\textwidth}
                 \centering
                 \includegraphics[width=0.98\textwidth]{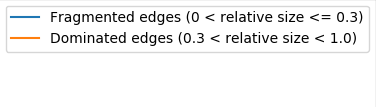}
                 \label{legend:bins}
                 \caption*{}
         \end{subfigure}
         \label{fig:bins}
\caption{Effect of iterative hyperedge reweighting: \% of hyperedges where the relative size of its largest partition falls in a given bin vs. no. of iterations}
\end{figure*}

 While \textit{Hypergraph-Louvain} is shown to run consistently faster than \textit{Hypergraph-Spectral} for the same number of nodes, the difference increases as the hypergraph grows larger. In Figure \ref{fig-runtime}, this is shown by the widening in the gap between the two curves as the number of nodes increases.

\section{Conclusion}

In this work, we have considered the problem of modularity maximization on hypergraphs. In presenting a modularity function for hypergraphs, we derived a node degree preserving graph reduction and a hypergraph null model. To make use of additional hyperedge information, we proposed a novel algorithm, Iteratively Reweighted Modularity Maximization (IRMM). This is based on a hyperedge reweighting procedure that refines the cuts induced by the clustering method. Empirical evaluations on real-world data and synthetic data illustrated the performance and scalability of our method.

The modularity function on graphs is based on a null model that preserves node degree alone. While we have incorporated additional information from the hyperedge cut into the clustering framework, other hyperedge-centric information and constraints are yet to be explored.
\bibliographystyle{ACM-Reference-Format}
\bibliography{references}

\end{document}